\documentclass[11pt]{article} %
\usepackage{times}
\usepackage{url}            %
\usepackage{booktabs}       %
\usepackage{amsfonts}       %
\usepackage{nicefrac}       %
\usepackage{microtype}      %
\usepackage{graphicx}
\usepackage{amsthm}
\usepackage{amssymb,amsmath,amsthm}
\newtheorem{theorem}{Theorem}

\usepackage{mkolar_definitions}
\usepackage{xspace}
\usepackage{multirow}
\usepackage{amsmath}
\usepackage{algorithm}
\usepackage{algorithmic}
\usepackage{color}
\usepackage{authblk}
\usepackage{color, colortbl}
\usepackage{enumitem}
\usepackage{comment}
\usepackage{bm}
\usepackage{subcaption}
\usepackage[left=2cm,top=2cm,right=2cm]{geometry}

\usepackage[scaled=.9]{helvet}
\usepackage{thmtools}

\usepackage{url}
\usepackage{authblk}
\usepackage{csquotes}
\usepackage{tikz}

\usepackage[colorlinks=true, linkcolor=blue, citecolor=blue]{hyperref}

\usepackage{natbib}

\usepackage{centernot}

\newcount\Comments  %
\definecolor{darkgreen}{rgb}{0,0.5,0}
\definecolor{darkred}{rgb}{0.7,0,0}
\definecolor{teal}{rgb}{0.3,0.8,0.8}
\definecolor{orange}{rgb}{1.0,0.5,0.0}
\definecolor{purple}{rgb}{0.8,0.0,0.8}
\newcommand{\kibitz}[2]{\ifnum\Comments=1{\textcolor{#1}{\textsf{\footnotesize #2}}}\fi}
\usetikzlibrary{positioning}

\definecolor{Gray}{gray}{0.9}

\newcommand{\mn}{\mathrm{mn}}
\newcommand{\pro}{\mathrm{pro}}
\newcommand{\both}{\mathrm{both}}

\newcommand{\ie}{\emph{i.e.}}
\newcommand{\eg}{\emph{e.g.}}

\usepackage{cleveref}

\begin{document}

\title{Minimax Instrumental Variable Regression 
and $L_2$ Convergence Guarantees
without Identification or Closedness}

\author[1]{Andrew Bennett \thanks{Alphabetical order }  }
\author[1]{Nathan Kallus}
\author[2]{Xiaojie Mao}
\author[3]{Whitney Newey} 
\author[4]{\\ Vasilis Syrgkanis }
\author[1]{Masatoshi Uehara \thanks{Corresponding: mu223@cornell.edu}} 

\affil[1]{Cornell University}
\affil[2]{Tsinghua University}
\affil[3]{Massachusetts Institute of Technology}
\affil[4]{Stanford University}
\maketitle

\maketitle

\begin{abstract}
{In this paper, we study nonparametric estimation of instrumental variable (IV) regressions. Recently, many flexible machine learning methods have been developed for instrumental variable estimation. However, these methods have at least one of the following limitations: (1) restricting the IV regression  to be uniquely identified; (2) only obtaining 
 estimation error rates in terms of pseudometrics (\emph{e.g.,} projected norm) rather than valid metrics (\emph{e.g.,} $L_2$ norm); or (3) imposing the so-called closedness condition that requires a certain conditional expectation operator to be sufficiently smooth. In this paper, we present the first method and analysis that can avoid all three limitations, while still permitting general function approximation. Specifically, we propose a new penalized minimax estimator that can converge to a fixed IV  solution even when there are multiple solutions, and we derive a strong $L_2$ error rate for our estimator under lax conditions. 
 Notably, this guarantee only needs a widely-used source condition and realizability assumptions, but not the so-called closedness condition. 
 We argue that the source condition and the closedness condition are inherently conflicting, so relaxing the latter significantly improves upon the existing literature that requires both conditions. 
 Our estimator can achieve this improvement because it builds on a novel formulation of the IV estimation problem as a constrained optimization problem. 
  }
\end{abstract}

\section{Introduction}

Instrumental variable (IV) estimation is an important problem in many applications. Examples
include causal inference \citep{angrist1995identification,newey2003instrumental,deaner2018proxy,cui2020semiparametric}, missing data problems \citep{wang2014instrumental,miao2015identification}, asset pricing models \citep{chen2014local,christensen2017nonparametric,escanciano2015nonparametric}, dynamic discrete choice models \citep{kalouptsidi2021linear}, and reinforcement learning \citep{liao2021instrumental,uehara2021finite}.

In this paper, we focus on the estimation of nonparametric IV (NPIV)  regression \citep{newey2003instrumental}.
This problem involves three sets of variables $X$, $Y$, and $Z$ that take values in compact Euclidean sets $D_X$, $D_Y$, and $D_Z$, respectively.
In the original IV estimation problem, $X$ stands for endogenous variables, $Y$ stands for an outcome variable, and $Z$ stands for exogenous IVs. 
We define $L_2(X),L_2(Z)$ as the $L_2$ spaces of functions of $X, Z$ respectively, defined in terms of their distributions. 
We are interested in solving the following equation with respect to $h \in L_2(Z)$: 
\begin{align*}
\EE\bracks{Y - h(X) \mid Z} = 0.
\end{align*}
This equation can be alternatively written as 
$\Tcal h = r_0$, where $r_0(Z) = \EE[Y \mid Z]$, and $\Tcal: L_2(X) \to L_2(Z)$ is a  bounded linear operator that maps every $h \in L_2(X)$ to $\EE\bracks{h(X) \mid Z} \in L_2(Z)$.
Here both the function $r_0$ and the operator $\Tcal$ are unknown. Instead, we only have access to a set of independent and identically distributed observations $\Dcal \coloneqq \{X_i,Y_i,Z_i\}_{i=1}^n$.

There has been a surge in interest in NPIV regressions. A number of classical works have proposed sieve or kernel-based estimators \citep[e.g.,][]{carrasco2007linear,horowitz2011applied,newey2013nonparametric,newey2003instrumental,chen2007large}.
However, NPIV estimation is notoriously difficult because it is an ill-posed inverse problem. 
In particular, the solution to the NPIV equation $\Tcal h = r_0$ may not be unique, and even if it is unique, the solution may depend on the underlying data distribution discontinuously \citep{carrasco2007linear}. 
Therefore, existing works typically assume that the NPIV solution is unique \citep{andrews2017examples,newey2003instrumental}. Even if it is not the case, they restrict the linear operator $\Tcal$ and the NPIV solution \citep{florens2011identification,chen2021robust}. A widely used restriction is the source condition, which assumes that the IV solution belongs to a subspace defined by the operator $\Tcal$ \citep[e.g.,][]{carrasco2007linear,cavalier2011inverse,chen2011rate}. 
Under these conditions, the estimators proposed in these classic literature can have strong theoretical guarantees. 
However, these traditional nonparametric estimators do not allow for the integration of modern, flexible general function approximation methods such as neural networks or tree-based methods.

To overcome this limitation, recent works have proposed various algorithms that can accommodate general function approximation. These algorithms typically employ two function classes, $\Hcal$ and $\Gcal$.
 In particular, the function class $\Hcal$ is the hypothesis class for the solution to the NPIV equation $\Tcal h=r_0$. 
 The  function class $\Gcal$, often referred to as a witness function class or discriminator class, is introduced to witness how much each given function $h$ violates the NPIV equation. Then, NPIV estimators are defined as solutions to a minimax optimization problem \citep{lewis2018adversarial,bennett2019deep,dikkala2020minimax,LiaoLuofeng2020PENE,muandet2020dual}:
\begin{align*}
    \argmin_{h \in \Hcal}\max_{g \in \Gcal}L(h,g)
\end{align*}
where $L(h,g)$ is an objective function mapping from $\Hcal \times \Gcal$ to $\mathbb{R}$. 

\begin{table}[!t]
    \centering
     \caption{Summary of current literature of minimax estimation with general function approximation. Our goal is to solve $\Tcal h= r_0$ with respect to $h$ with unknown $\Tcal$ and $r_0$. We denote its set of solutions by $\Hcal_0$ and the 
      least norm solution by $h_0$. 
      Estimators are defined as solutions to certain minimax optimizations $\min_{h\in \Hcal}\max_{g\in \Gcal}L(h,g)$ where $\Hcal$ and $\Gcal$ are hypothesis classes. For simplicity, we focus on comparison for VC classes $\Hcal$ and $\Gcal$ (while the results in both our and these papers deal with general function classes) and for source condition with exponent 1. 
      We let $\Tcal^{*}$ be the adjoint of $\Tcal$, $\bar\Gcal_0=\{\bar g_0:\Tcal^{*}\bar g_0 =h_0\}$ (nonempty under source condition), $h_{0,\alpha}=\argmin_{h}\|\Tcal h-r_0\|^2_2 +\alpha \|h\|^2_2$, and $\alpha_0>0$ a positive number. Note our condition is strictly weaker than that of \citet{bennett2022inference}. 
}
 \resizebox{1.0\textwidth}{!}{
    \label{tab:summary_whole}
    \begin{tabular}{cccc}
    \toprule
         &  Primary assumptions & Guarantee  & Rate  \\ \midrule
    \citet{dikkala2020minimax}  & realizability $\Hcal_0 \cap \Hcal\neq \emptyset$, closedness $\Tcal \Hcal \subset \Gcal+r_0$ & Projected MSEs & $n^{-1/2}$\\  
    \citet{LiaoLuofeng2020PENE}  &  
    \begin{tabular}{c}
       source $h_0 \in \Rcal(\Tcal^{*}\Tcal)$ , uniqueness of $h_0$,  \\ 
       realizability $ h_{0,\alpha} \in \Hcal\;\forall\alpha\leq\alpha_0$, closedness $ \Tcal \Hcal \subset \Gcal+r_0$   
    \end{tabular}
    &  $L_2$ rates  &  $n^{-1/6}$ \\ 
    \citet{bennett2022inference}   & \begin{tabular}{c}
       source $h_0 \in \Rcal(\Tcal^{*}\Tcal)$,    \\ 
       realizability $h_0 \in \Hcal
       ,\,\bar \Gcal_0 \cap \Gcal \neq \emptyset
       $ and closedness $ \Tcal^{*} \Gcal \subset \Hcal$   
    \end{tabular}   & $L_2$ rates &  $n^{-1/4}$  \\   \rowcolor{Gray}
    This work    &  \begin{tabular}{c} source $h_0 \in \Rcal(\Tcal^{*}\Tcal)$ \\ 
          realizability $h_0 \in \Hcal
          ,\,\bar \Gcal_0\cap \Gcal \neq \emptyset
          $  
              \end{tabular}  
  & $L_2$ rates  & $n^{-1/4}$ \\ 
      \bottomrule
    \end{tabular}
    }

\end{table}

Although highly flexible, these  minimax approaches have several limitations. 
First, they typically assume that 
the solution to the NPIV equation $\Tcal h=r_0$ is unique.  
However, this assumption can be easily violated if the instrumental variables are not very strong \citep{andrews2005inference,Andrews2019weak}, and they usually do not hold in proximal causal inference \citep{kallus2021causal}. 
Secondly, the minimax estimators may not give strong $L_2$ error rate guarantees, and instead only have error rate guarantees in terms of a weaker projected mean squared error (MSE) \citep{dikkala2020minimax}. 
However, even when the projected MSE vanishes to zero, the minimax estimator may not converge to any fixed IV solution since the projected MSE is a pseudometric unlike the $L_2$ metric. %
Thirdly, current minimax estimators typically need some form of closedness condition, such as $\Tcal h \in \Gcal$ for any $h \in \Hcal$ \citep{dikkala2020minimax,LiaoLuofeng2020PENE} or other close variant 
 \citep{bennett2022inference}.
However, this assumption may impose stringent restrictions on the operator $\Tcal$, noting that $\Gcal$ must be a restricted class to ensure bounded statistical complexity. In particular, the closedness assumption is at odds with the widely used source condition, since we will show that the closedness assumption is more plausible when the spectrum of $\Tcal$ decays more slowly while the source condition is more plausible when the spectrum decays more rapidly.

To the best of our knowledge, all current approaches incorporating general function approximation for IV problems suffer from at least one of the three limitations listed above. In this paper, we propose the first method that avoids all three of these limitations. Specifically, we do not assume that the NPIV solution is unique, and instead we target the least norm solution $h_0$. 
This is a standard approach for inverse problems with non-unique solutions \citep{florens2011identification,babii2017completeness,chen2021robust,bennett2022inference}.
We show that our proposed estimator can converge to the least norm IV solution and derive its $L_2$ error rate guarantee.
These theoretical guarantees only need the fairly standard source condition and realizability assumptions (\ie, well-specification of $\Hcal$ and $\Qcal$). \Cref{tab:summary_whole}  summarizes the assumptions and  guarantees in our paper and related ones.

Our proposed estimator and its theory are grounded in the novel insight that finding the least norm solution $h_0$ to $\Tcal h=r_0$ can be viewed as a constrained optimization problem. 
In particular, we show that the least norm solution 
can be uniquely identified as a saddle point of the minimax optimization of the Lagrangian.  
Although previous minimax estimators also leverage minimax optimization, their  inner maximization is used to approximate the projected MSE $\EE[{\prns{[\Tcal h](Z) - r_0(Z)}^2}]$, which necessitates the closedness assumption. 
In contrast, the inner maximization in our methods results from the method of Lagrange multipliers, and it does not need the closedness assumption. 
Interestingly, we prove that the source condition is the sufficient and necessary condition for the existence of stationary Lagrange multipliers and thus the saddle point to our minimax optimization problem. 
This also reveals a new role of the  source condition widely used in inverse problems.

Our paper is organized as follows. In \pref{sec:setup}, we present our setup of IV estimation and the limitations of current works in this setting. In \pref{sec:algorithm}, we introduce our minimax estimator by framing the problem as a constrained optimization problem. In \pref{sec:primary_identification}, we demonstrate that the minimax optimization identifies the least norm solution given infinite data. In \pref{sec:finite}, we present the finite-sample error guarantee, i.e., $L_2$ convergence rate. In \pref{sec:discussion}, we compare our estimator and theory to those in closely related works. Finally, we conclude our paper in \pref{sec:conclusion}.

\subsection{Related Works}\label{subsec:related}

Instrumental variable estimation has received considerable attention as a subclass of inverse problems, as detailed in the works of \citet{carrasco2007linear,cavalier2011inverse, newey2013nonparametric,ito2014inverse}.

Even when the operator $\Tcal$ and response $r_0$ are known, nonparametric instrumental variable estimation poses significant difficulties due to its ill-posed nature. The ill-posedness often refers to the presence of one or more of the following characteristics: (1) the absence of solutions, (2) the existence of multiple solutions, and (3) the discontinuity of the inverse of $\Tcal$. To address these challenges, various regularization techniques have been proposed, such as compactness of the solution space \citep{newey2003instrumental}, Tikhonov regularization, and Landweber–Fridman regularization \citep{carrasco2007linear,cavalier2011inverse}. In practical settings where $\Tcal$ and $r_0$ are unknown, a range of estimators have been proposed in the literature, including series-based estimators \citep{ai2003efficient,hall2005nonparametric,blundell2007semi,chen2011rate,darolles2011nonparametric,chen2012estimation,florens2011identification,chen2021robust}, kernel-based estimators \citep{hall2005nonparametric,horowitz2007asymptotic}, and RKHS-based estimators \citep{singh2019kernel,muandet2020dual}.

Recently, there has been growing interest in the application of general function approximation techniques, such as deep neural networks and random forests, to instrumental variable problems in a unified manner \citep{dikkala2020minimax,lewis2018adversarial,bennett2019deep,zhang2020maximum}. Among these approaches, \citet{dikkala2020minimax,LiaoLuofeng2020PENE,bennett2022inference} provide finite-sample convergence rate guarantees. Specifically, \citet{LiaoLuofeng2020PENE} establishes $L_2$ convergence by linking minimax optimization with Tikhonov regularization under the assumption of the source condition. \citet{bennett2022inference} establishes an $L_2$ convergence guarantee under the source condition from a distinct perspective. Notably, the assumptions we need are strictly weaker than those of \citet{bennett2022inference}. \citet{dikkala2020minimax} guarantees convergence in terms of projected mean squared error without the source condition; however, this guarantee is insufficient to identify a specific element when the solution is not unique. These works \citep{dikkala2020minimax,LiaoLuofeng2020PENE,bennett2022inference} rely on the so-called closedness assumption, which imposes restrictions on the smoothness of the operator $\Tcal$ via the witness class. This assumption has been the subject of considerable discussion in the context of offline reinforcement learning, with researchers exploring ways to relax it \citep{chen2019information,uehara2020minimax,foster2021offline,huang2022beyond}. In this paper, we examine the relaxation of this assumption in a more general IV setting. { This is of importance since the source condition and closedness are inherently conflicting.  }

We note that there are a number of alternative approaches for integrating machine learning into instrumental variable estimation \citep{hartford2017deep,yu2018deep,xu2020learning,liu2020deep,kato2021learning,lu2021machine}. However, to the best of our knowledge, these approaches do not offer an $L_2$ convergence rate guarantee in the absence of the assumption of uniqueness.

\section{Problem Setup}\label{sec:setup}

We aim to solve the following equation with respect to $h$: 

\begin{align}\label{eq:data}
 \textstyle    \Tcal h = r_0
\end{align}

where  $r_0(Z) = \EE[Y\mid Z]\in L_2(Z)$ is an unknown function and 
$\Tcal: L_2(X) \to L_2(Z)$ is an unknown conditional expectation operator that maps any $h \in L_2(X)$ to $\EE[h(X)\mid Z]$.
Note that $\Tcal$ is a bounded  operator since its norm  is upper-bounded by 1 via Jensen's inequality. 
Moreover, we use $\Tcal^{\star}: L_2(Z) \to L_2(X)$ to denote the adjoint operator of $\Tcal$, \ie,  $\langle g,\Tcal h\rangle_{L_2(Z)}=\langle \Tcal^{\star}g,h\rangle_{L_2(X)}$ for any $h \in L_2(X), g \in L_2(Z)$ where $\langle \cdot,\cdot \rangle_{L_2(X)}$ and $\langle \cdot,\cdot \rangle_{L_2(Z)}$ are inner products over $L_2(X)$ and $L_2(Z)$, respectively. It is known that $\Tcal^{*}$ is given by $[\Tcal^* g](X) = \EE[g(Z)\mid X]$ for any $g \in L_2(Z)$ \citep{carrasco2007linear}. 
Importantly, here we do not assume compactness of $\Tcal$, because compactness is violated whenever $X, Z$ include common variables, as is the case in many applications \citep{deaner2018proxy,cui2020semiparametric}. 
Moreover, we denote the range space of $\Tcal$ by $\Rcal(\Tcal)$, i.e., $\Rcal(\Tcal)=\{\Tcal h: h \in L_2(X)\}$. 

Throughout this work, we assume that there exists a solution to \Cref{eq:data}.

\begin{assum}[Existence of solutions]\label{assum:exist}
We have $r_0 \in \Rcal(\Tcal)$, i.e., $\Ncal_{r_0}(\Tcal)\coloneqq \{h\in \Hcal: \Tcal h=r_0\} \ne \emptyset$. 
\end{assum}

Most of the existing literature further assumes that $\Tcal$ is injective and the solution to \Cref{eq:data} is unique. 
However, even in this case, \Cref{eq:data} still corresponds to an ill-posed inverse problem, since the inverse operator $\Tcal^{-1}$ is generally unbounded, so the NIPV solution can be very sensitive to even slight perturbations to the data distributions. 
Without further restrictions, we can only obtain an estimator $\hat h$ with convergence guarantee in terms of the projected MSE $\EE[\{\Tcal \hat h-r_0\}^2(Z)] = \EE[\{\Tcal (\hat h-h)\}^2(Z)]$ for  $h\in \Ncal_{r_0}(\Tcal)$.  
However, the projected MSE is only a pseudometric. Hence, even if $\EE[\{\Tcal(\hat h-h)\}^2(Z)]$ vanishes to zero, the estimator $\hat h$ may not converge to a fixed point. Furthermore, the projected MSE is weaker than the valid metric such as the $L_2$ metric. 
Indeed, according to Jensen's inequality, we have 
$\EE[\{\hat h-h)\}^2(X)]\geq \EE[\{\Tcal(\hat h-h)\}^2(Z)]. $
However, the other direction generally does not hold. Thus $\EE[\{\hat h-h)\}^2(X)]$ may not vanish even when $\EE[\{\Tcal(\hat h-h)\}^2(Z)]$ does. 

In many problems, $L_2$ rate guarantees are preferable or even necessary \citep{hall2005nonparametric,chen2011rate,kallus2021causal,uehara2021finite}. 
In order to achieve $L_2$ convergence, we need to further restrict the ill-posedness of the NPIV problem. 
One common way is to restrict the magnitude of the ill-posedness measure $\sup_{h \in \Hcal}\frac{ \EE[\{h-h')\}^2(X)] }{\EE[\{\Tcal(h-h')\}^2(Z)]}$ for any solution $h' \in \Ncal_{r_0}(\Tcal)$, where 
$\Hcal$  is  the function class used to obtain the estimator \citep[\eg, ][]{dikkala2020minimax,chen2012estimation}. 
This allows us to translate projected MSE guarantees to corresponding $L_2$ error rates under the uniqueness of \Cref{eq:data}.

However, in this paper, we do not assume a unique solution to \Cref{eq:data}, because it may not hold in many practical settings. 
In particular, uniqueness is   violated when instrumental variables are weak \citep{andrews2005inference,Andrews2019weak}. 
For instance, when the spaces $D_X$ and $D_Z$ are discrete and the cardinality of $D_Z$ exceeds that of $D_X$,  uniqueness generally does not  hold. 
Moreover, uniqueness is usually violated in proximal causal inference, as \cite{kallus2021causal} demonstrates in various examples.
When solutions are non-unique, \Cref{eq:data} becomes even more ill-posed.  
In this case, existing estimators may still have projected MSE guarantees, but obtaining $L_2$ rate guarantees becomes much more difficult. 
In particular, the ill-posedness measure  is generally infinity and thus uninformative. 
Most of the existing estimators do not necessarily converge to any particular solution in $\Ncal_{r_0}(\Tcal)$ in terms of the $L_2$ metric.

Given that there may be (infinitely) many solutions in $\Ncal_{r_0}(\Tcal)$, we propose to target a particular solution that achieves the least norm, that is,
\begin{align}\label{eq: min-norm}
h_0 = \argmin_{h \in \Ncal_{r_0}(\Tcal)} ~ 0.5\langle h, h\rangle_{L_2(X)}. 
\end{align}
This least norm solution is well-defined as it is the projection of the origin in $L_2(X)$ onto a closed affine space $\Ncal_{r_0}(\Tcal) \subseteq L_2(X)$. We formalize this in the following lemma. 
\begin{lemma}\label{lem:minimum}
Suppose Assumption~\ref{assum:exist} holds. Then the least norm solution $h_0 \in \Ncal_{r_0}(\Tcal)$ uniquely exists, and $\{h_0\}= \overline{\Rcal(\Tcal^{\star})} \cap \Ncal_{r_0}(\Tcal)$, where $ \overline{\Rcal(\Tcal^{\star})}$ is the closure of the range space ${\Rcal(\Tcal^{\star})}$. 
\end{lemma}   

We note that some of the existing literature also targets the least norm solution when the IV equation admits non-unique solutions \citep{florens2011identification,santos2011instrumental,chen2021robust}, but they all focus on classic sieve or kernel-based estimators. 
The only exception is \cite{bennett2022inference} as they employ general function approximation while allowing for non-unique solutions. But as we discuss in \Cref{sec: compare-bennett}, their method requires a  closedness assumption that puts strong restrictions on the operator $\Tcal$. In this paper, we propose a new estimator for the least norm solution $h_0$ with a strong $L_2$ convergence guarantee. 
Importantly, our estimator 
accommodates general function approximation but does not need the closedness assumption, thereby improving upon the existing literature.

\section{Penalized Minimax Instrumental Variable Regression}\label{sec:algorithm}

In this section, we propose our estimator for the least norm solution $h_0$ in \Cref{eq: min-norm}. To this end, we first provide a reformulation of the solution $h_0$. Note that 
\begin{align*}
h_0 = \argmin_{h \in L_2(X)} 0.5\langle h, h\rangle_{L_2(X)}, ~ \text{subject to } \Tcal h = r_0.
\end{align*}
This is a constrained optimization problem over the Hilbert space $L_2(X)$. Following the method of Lagrange multipliers, we can  consider 
an alternative minimax optimization: 
\begin{align}\label{eq:motivatioin2}
     h_0 = \argmin_{h \in L_2(X) }\sup_{g \in L_2(Z)}L(h,g),\quad L(h,g)\coloneqq 0.5 \langle h,h\rangle_{L_2(X)}+\langle r_0-\Tcal h,g\rangle_{L_2(Z)}, 
\end{align}
where $g$ corresponds to a Lagrange multiplier. 

In \Cref{eq:motivatioin2}, the objective function $L(h, q)$ is unknown since the two inner products involve the unknown function $r_0$,  the unknown operator $\Tcal$, and the unknown distribution of $X$ and $Z$.
To construct an estimator based on \Cref{eq:motivatioin2}, we first rewrite the inner products into expectations with respect to the distribution of $X$ and $Z$: 
\begin{align*}
\langle h,h\rangle_{L_2(X)} = \EE\bracks{h^2(X)}, ~~ \langle r_0-\Tcal h,g\rangle_{L_2(Z)} = \EE\bracks{\prns{Y - h(X)}g(Z)}. 
\end{align*}
Then we can replace the unknown expectations with empirical averages, and restrict the functions $h$ and $g$ to some classes $\Hcal\subset [D_X \to \RR]$ and $\Gcal \subset [D_Z\to \RR]$. 
This leads to the following estimator: 
\begin{align}\label{eq:motivatioin3}
 \hat h_{\mn} \in \argmin_{h \in \Hcal }\max_{g \in \Gcal} L_n(h,g),\quad  L_n(h,g)\coloneqq 0.5 \EE_n[h^2(X)]+\EE_n\bracks{\prns{Y-h(X)}g(Z)},
\end{align}
where $\EE_n\bracks{\cdot}$ stands for the empirical average operator based on sample data $\Dcal=\{X_i,Y_i,Z_i\}$. For example, we have $\EE_n[h^2(X)] = \frac{1}{n}\sum_{i=1}^n h^2(X_i)$. 
Notably, the term $\EE_n[h^2(X)]$ in \Cref{eq:motivatioin3} can be viewed as a penalization term, so we call our estimator a penalized minimax estimator. The role of this penalization term is later discussed in \pref{thm:idenfication}. 

The estimator $\hat h_{\mn}$ in \Cref{eq:motivatioin3} has a minimax optimization formulation. The computational perspective will be discussed in \pref{sec:comptuational}. This is in line with many recent machine learning IV estimators with general function approximation (see a review in \Cref{subsec:related}). 
However, our minimax optimization in \Cref{eq:motivatioin3} is motivated by the method of Lagrange multipliers, while existing minimax estimators are based on fundamentally different principles. 
As a result, our objective function $L_n(h, g)$ differs from those used in existing minimax estimators. 
In particular, our minimax estimator requires quite different conditions, as we will discuss in \Cref{sec:discussion}.

To justify the objective function in \pref{eq:motivatioin3}, we need to further guarantee that  
\begin{align}\label{eq:goal}
    h_0  = \argmin_{h \in \Hcal}\max_{g \in \Gcal}L(h,g).
\end{align}
In Section~\ref{sec:primary_identification}, we establish \Cref{eq:goal} under fairly mild conditions. 
Based on this, we then further derive the $L_2$ convergence rate of our proposed estimator $\hat h_{\mn}$.

\section{Identification of the Least Norm Solution}\label{sec:primary_identification}

In this section, we establish  that our proposed minimax formulation can indeed identify the least norm solution $h_0$ as shown in \Cref{eq:goal}. 
We start with introducing a  key assumption for our result, and then present our identification result under this assumption.

\subsection{Source Condition}

Our identification crucially depends on the following source condition. 
\begin{assum}[Source condition]\label{assum:source}
The function $r_0$ satisfies that $r_0\in \Rcal(\Tcal\Tcal^{\star})$.     
\end{assum}

Assumption~\ref{assum:source} further strengthens Assumption~\ref{assum:exist} in that it restricts $r_0$ to a smaller subspace $\Rcal(\Tcal\Tcal^{\star}) \subseteq \Rcal(\Tcal)$. 
In particular, we have $\Rcal(\Tcal) = \Tcal(\overline{\Rcal({\Tcal^{\star}})})$\footnote{To see this note that $L_2(X) = \overline{\Rcal(\Tcal^\star)} \oplus \Rcal(\Tcal^\star)^\perp$, and $\Rcal(\Tcal^\star)^\perp = \Ncal(\Tcal)$, so $\Tcal(L_2(X)) = \Tcal(\overline{\Rcal(T^\star)})$.} and $\Rcal(\Tcal \Tcal^\star) = \Tcal(\Rcal({\Tcal^{\star}}))$, so $\Rcal(\Tcal \Tcal^\star)$ is generally a \emph{strict} subset of  $\Rcal(\Tcal)$,  unless $\Rcal(\Tcal^\star)$ is a closed set.
It is well known that for ill-posed inverse problems, the operator $\Tcal^\star$ generally does not have a closed range space \citep{carrasco2007linear}, thus in general Assumption~\ref{assum:source}  imposes non-trivial restrictions on the ill-posedness of the inverse problem. 
In \Cref{subsec:tension}, we provide a more concrete example to illustrate these restrictions. 

Source conditions are common assumptions used to derive strong convergence rate guarantees in the inverse problem literature.
They have been 
widely used for both inverse problems with known operators \citep[\eg, ][]{engl1996regularization,ito2014inverse} and IV problems with unknown operators \citep[\eg, ][]{florens2011identification,carrasco2007linear,liao2021instrumental}. 
A standard source condition in the literature is that the solution $h_0$ satisfies $h_0 \in \Rcal((\Tcal\Tcal^{\star})^{\beta/2})$ for a positive exponent $\beta > 0$. 
Our source condition in Assumption~\ref{assum:source} can be shown to be equivalent to $h_0 \in \Rcal((\Tcal\Tcal^{\star})^{1/2})$ via the spectral theory of linear operators \citep{cavalier2011inverse}. 
Thus, our Assumption~\ref{assum:source} is a source condition of this kind with source exponent $\beta = 1$. 

Assumption~\ref{assum:source} implies that  
there exists $\bar{g}_0 \in L_2(Z)$ such that 
\begin{align}\label{eq:important}
r_0 = \Tcal\Tcal^\star \bar{g}_0.
\end{align}
In fact, any $\bar g_0$ satisfying \Cref{eq:important} is closely related to the least norm solution $h_0$.

\begin{lemma}\label{lem:norm}
If Assumption~\ref{assum:source} holds, then $\bar g_0$ satisfies \eqref{eq:important} if and only if $\Tcal^\star \bar g_0 = h_0$.
\end{lemma}

In particular, given \Cref{lem:norm}, the functions $\bar g_0$ that satisfy \Cref{eq:important} are given by:
\begin{align}\label{eq: g-funs}
     \Ncal_{h_0}(\Tcal^{\star}) \coloneqq \{g\in L_2(Z) : \Tcal^{\star}g = h_0 \}. 
\end{align}

In the next subsection, we will show the importance of the source condition given by Assumption~\ref{assum:source}. 
In particular, this condition ensures that we can obtain $h_0$ from the saddle points of $L(h,g)$.

\subsection{Saddle Points of the Minimax Optimization}\label{subsec:saddle}

Here, we characterize the saddle points of $L(h,g)$ under Assumption~\ref{assum:source}, as follows:

\begin{lemma}\label{lem:saddle_point}
Suppose Assumption~\ref{assum:source} holds and let $h_0$ be the least norm solution in \Cref{eq: min-norm} and $\Ncal_{h_0}(\Tcal^{\star})$ be the set of functions given in \Cref{eq: g-funs}. 
Then, the set of saddle points of $L(h,g)$ over $h\in L_2(X), g\in L_2(Z)$, i.e., the points $(h',g')$ that satisfy
\begin{align*}
    L(h,g') \geq L(h',g') \geq  L(h',g),~ \forall h\in L_2(X), \forall g\in L_2(Z), 
\end{align*}
is given by the set $\braces{h_0} \times \Ncal_{h_0}(\Tcal^{\star}) = \braces{(h_0, \bar g): \bar g \in \Ncal_{h_0}(\Tcal^{\star})}$.
\end{lemma}

It is well-known that $(h',g')$ is a saddle point if and only if we have the ``strong duality'' condition\\ $ \inf_{h \in L_2(X)} \sup_{g \in L_2(Z)}L(h,g) = \sup_{g \in L_2(Z)} \inf_{h \in L_2(X)} L(h,g)$ and 
\begin{align*}
  h' \in  \argmin_{h \in L_2(X)} \sup_{g \in L_2(Z)}L(h,g),\quad   g'\in \argmax_{g \in L_2(Z)} \inf_{h \in L_2(X)} L(h,g).
\end{align*}
We provide formal proof for this in 
\pref{sec:auxiliary}. Given this equivalent characterization of the saddle point,
we can obtain the following corollary from 
 \pref{lem:saddle_point}.

\begin{corollary}\label{corollary: saddle}
{If Assumption~\ref{assum:source} holds, } then we have 
\begin{align}\label{eq:immediate_consequence}
    h_0 = \argmin_{h \in L_2(X)} \sup_{g \in L_2(Z)}L(h,g),\quad \Ncal_{h_0}(\Tcal^{\star}) = \argmax_{g \in L_2(Z)} \inf_{h \in L_2(X)} L(h,g). 
\end{align}
\end{corollary}

It is worth noting that the equality for $h_0$ in \Cref{eq:immediate_consequence} holds even without the source condition.
Moreover, the strong duality $ \inf_{h \in L_2(X)} \sup_{g \in L_2(Z)}L(h,g) = \sup_{g \in L_2(Z)} \inf_{h \in L_2(X)} L(h,g)$ also holds in the absence of this source condition.
However, the source condition is important to establish the existence of $\argmax_{g \in L_2(Z)} \inf_{h \in L_2(X)} L(h,g)$ and the second statement in \eqref{eq:immediate_consequence}. Equivalently, this shows that the source condition guarantees the existence of stationary Lagrangian multipliers  for the  problem in \Cref{eq: min-norm}, and the set of stationary Lagrangian multipliers is given by  $\Ncal_{h_0}(\Tcal^{\star})$.

So far we have demonstrated that Assumption~\ref{assum:source} is a sufficient condition for the existence of saddle points. Interestingly, it is also a necessary condition for their existence.

\begin{lemma}\label{lem:saddle_point2}
Suppose Assumption~\ref{assum:exist} that $r_0 \in \Rcal(\Tcal)$ holds. Then, there exists a saddle point of $L(h,g)$ if and only if Assumption~\ref{assum:source} holds. 
\end{lemma}

The above lemma is proved by first showing that the saddle point exists if and only if there exists a solution to $   \argmin_{g \in L_2(Z)}\|\Tcal^{*}g-h_0\|^2_2$. We then demonstrate that the existence of this optimization problem
is equivalent to 
 the source condition (\ref{assum:source}).
 Our 
\pref{lem:saddle_point} and \pref{lem:saddle_point2} show that the source condition is closely  related to the existence of stationary Lagrangian multipliers for the constrained optimization formulation of $h_0$.
To our knowledge, this relation is novel in the literature. 

\pref{lem:saddle_point} characterizes the saddle points over the unrestricted $L_2(X)$ and $L_2(Z)$ spaces. 
However, in practical estimation, we can only use some function classes $\Hcal \subset L_2(X),\Gcal \subset L_2(Z)$ with limited statistical complexity.
For these two restricted classes to capture some saddle points, we need them to  satisfy the following  realizability assumptions. 

\begin{assum}[Realizability of the least norm solution]\label{assum:realizability1}
    We have $h_0 \in \Hcal$. 
\end{assum}

\begin{assum}[Realizability of the stationary Lagrange multiplier]\label{assum:realizability2}
We have $\Ncal_{h_0}(\Tcal^{\star})\cap \Gcal \neq 0$. 
\end{assum}

The  realizability assumptions above require that the function classes $\Hcal$ and $\Gcal$ are well-specified, in that they  contain at least some true saddle points. In particular, Assumption~\ref{assum:realizability2} is equivalent to $h_0 \in \Tcal^{*}\Gcal$. 
In the following theorem, we further extend the saddle point characterization of $h_0$ in \Cref{corollary: saddle} to these restricted classes under these realizability conditions.

\begin{theorem}[Key identification theorem]\label{thm:idenfication}
Suppose Assumption~\ref{assum:source},\ref{assum:realizability1},\ref{assum:realizability2} hold. 
Then 
\begin{align*}
h_0=\argmin_{h \in \Hcal}\max_{g \in \Gcal}L(h,g).%
\end{align*}
\end{theorem}

\Cref{thm:idenfication} shows that under the source condition and the realizability assumptions, the min-max optimization of our proposed objective over the function classes $\Hcal, \Gcal$ can recover the saddle points in the classes.
At a high level, the proof of this theorem works by showing: (1) saddle points over the original class remain saddle points over the restricted classes; (2) any additional saddle points over the restricted classes are best-responses to saddle points over the original class; and (3) $h_0$ is a unique best response to any $\bar g \in \Ncal_{h_0}(\Tcal^{\star})$ as a result of  strong convexity of $L(h,g)$ in $h$ induced by $\langle h,h\rangle_{L_2(X)}$.
See \pref{sec:general} for details.

\section{Finite Sample Guarantees}\label{sec:finite}

As discussed in \pref{sec:primary_identification}, our proposed minimax optimization formulation can identify the target least norm solution $h_0$ when the population distribution is known. 
In this section, we further show that our finite-sample estimator $\hat h_{\mn}$ in \Cref{eq:motivatioin3} converges to $h_0$, and we derive its $L_2$ error rate. 

\begin{theorem}[$L_2$ convergence rates]\label{thm:main}
Suppose Assumption~\ref{assum:source},\ref{assum:realizability1},\ref{assum:realizability2} hold. Then, we have 
\begin{align*}
    \| \hat h_{\mn}-h_0\|_2 &\leq   \sqrt{2\sup_{h\in \Hcal,g\in \Gcal } \Big|(\EE_n-\EE)[\prns{Y-h(X)}g(Z) + 0.5 h(X)^2] \Big| }. %
\end{align*}    
\end{theorem}

Notably, the assumptions required in  \pref{thm:main} are identical to those in \Cref{thm:idenfication}. 
In particular, both theorems only require that the function classes $\Hcal$ and $\Qcal$ satisfy the realizability conditions Assumption~\ref{assum:realizability1},\ref{assum:realizability2}. 
Realizability is a fundamental assumption in statistical learning theory. For instance,
realizability is a standard assumption in least squares regression problems \citep[\eg,][]{wainwright2019high}.
To the best of our knowledge, existing minimax IV regression  estimators additionally require much stronger conditions such as $\Tcal \Hcal \subset \Gcal$ or $\Tcal^{*}\Gcal \subset \Hcal$. 
These conditions are often referred to as the closedness condition, and they impose additional restrictions on the operator $\Tcal$. See  \Cref{sec:discussion} for a detailed discussion.  

It then remains to bound the right-hand side term in \pref{thm:main}. 
This is an empirical process term, which can be easily upper-bounded by invoking standard statistical learning theory for any reasonable function classes $\Hcal,\Gcal$ with bounded statistical complexities. 
In particular, we can use 
standard symmetrization arguments to bound the right-hand side of \pref{thm:main} with the Rademacher complexities of $\Hcal, \Gcal$. 
The Rademacher complexity $\mathfrak{R}_n(\Hcal)$ of class $\Hcal$ is defined as $\mathfrak{R}_n(\Hcal) = n^{-1} \EE[\sup_{h\in \Hcal}\sum_{i=1}^n \sigma_i h(X_i)]$ 
 where $\{\sigma_1, \dots, \sigma_n\}$ are independent random variables drawn from the Rademacher distribution.
The Rademacher complexity  $\mathfrak{R}_n(\Gcal)$ of class $\Gcal$ can be defined analogously.

\begin{corollary}\label{cor:rademacher}
Suppose Assumption~\ref{assum:source},\ref{assum:realizability1},\ref{assum:realizability2} hold. Let $\|Y\|\leq C_{Y}, \|h\|_{\infty}\leq C_{\Hcal}$ for any $h \in \Hcal$ and $\|g\|_{\infty}\leq C_{\Gcal}$ for $g\in \Gcal$. %
Then, there exists a universal positive constant $c$ such that with probability at least $1-\delta$, we have 
\begin{align*}
    \| \hat h_{\mn}-h_0\|_2 \leq  c\sqrt{ (C_{\Hcal} + C_{\Gcal})(\mathfrak{R}_n(\Gcal) +\mathfrak{R}_n(\Hcal))+ (C_{\Gcal}+C_{\Hcal})C_{\Hcal}\sqrt{\ln(1/\delta)/n} }   
\end{align*}
\end{corollary}

Furthermore, for given function classes $\Hcal, \Qcal$, we can obtain final $L_2$ convergence rates by plugging in  off-the-shelf results of Rademacher complexities. For example, the following corollary is obtained by instantiating \pref{thm:main} to finite classes.

\begin{corollary}\label{cor:iv2}
When $\Hcal,\Gcal$ are finite classes, with probability at least $1-\delta$, we have $
         \| \hat h_{\mn} -h_0\|_2= \mathrm{Poly}(C_{\Hcal},C_{\Gcal})\prns{\frac{\ln\prns{|\Hcal||\Gcal|/\delta}}{n}}^{1/4}$
    where $\mathrm{Poly}(C_{\Hcal},C_{\Gcal})$ is a polynomial term in $C_{\Hcal}$ and $C_{\Gcal}$. 
\end{corollary}

As another example, we instantiate \Cref{thm:main} for more general nonparametric classes whose complexity are characterized by their covering numbers.

\begin{corollary}\label{cor:iv}
Let $M(\epsilon,\Hcal,\|\cdot\|_{\infty})$ and $M(\epsilon,\Gcal,\|\cdot\|_{\infty})$ be covering numbers of $\Hcal,\Gcal$ with respect to $L^{\infty}$-norm.  Suppose 
$\ln(M(\epsilon,\Hcal,\|\cdot\|_{\infty}))=O(\epsilon^{-\beta})$ and $\ln(M(\epsilon,\Gcal,\|\cdot\|_{\infty}))=  O(\epsilon^{-\beta})$ for some $\beta>0$, and the conditions of in \Cref{cor:rademacher} hold. 
Then with probability at least $1-\delta$, we have
\begin{align*}
    \| \hat h_{\mn} -h_0\|_2 =  \begin{cases}
       \mathrm{Poly}(C_{\Hcal},C_{\Gcal})\{n^{-1/4}  + (\ln(1/\delta)/n)^{1/4}\}, \quad (\beta<2) \\ 
              \mathrm{Poly}(C_{\Hcal},C_{\Gcal})\{n^{-1/4}\ln(n)  + (\ln(1/\delta)/n)^{1/4}\},  \quad (\beta=2)\\ 
       \mathrm{Poly}(C_{\Hcal},C_{\Gcal})\{n^{-1/(2\beta)}  + (\ln(1/\delta)/n)^{1/4}\} \quad (\beta>2).
    \end{cases}
\end{align*}    

\end{corollary}

If we specialize \Cref{cor:iv} to Sobolev balls $\Hcal,\Gcal$ with smoothness parameter $\alpha$ and input dimension $d$,
we have $\beta = d/\alpha$, so 
 the rates become $O(n^{-1/4})$ when $\alpha/d > 2$ and $O(n^{-\alpha/(2d)})$ when $\alpha/d \leq 2$. 
It is an interesting question whether this rate is optimal. 
Although \citet{chen2011rate} derives a minimax rate for NPIV regression estimation, their result requires the NIPV equation to have a unique solution and they impose stronger conditions on the function classes, so it is not directly comparable to our rate.  
 A thorough investigation of the rate optimality is left for future work.

Finally, we also consider the case where the function classes $\Hcal, \Gcal$ are misspecified so they may not satisfy the realizability assumptions. This result is useful when we use sieve estimators based on  sample-dependent function classes $\Hcal$ and $\Gcal$, that approximate certain function spaces. For example, $\Hcal, \Gcal$ can be linear models with polynomial basis functions or neural networks with growing dimensions, which can gradually approach H\"older or Sobolev balls \citep{chen2007large}. 

\begin{theorem}[Finite sample result under misspecification]\label{thm:misspecification}
Suppose Assumption~\ref{assum:source} holds, and there exists $h^\dagger \in \Hcal$ and $g^\dagger \in \Gcal$ such that $\|h^\dagger - h_0\|_2 \leq \epsilon_h$ and $\inf_{\bar g_0 \in \Ncal_{h_0}(\Tcal^\star)} \|g^\dagger - \bar g_0\|_2 \leq \epsilon_g$. 
Then
\begin{align*}
   \|\hat h_{\mn}-h_0\|_2 \leq \sqrt{  \{2C_{\Hcal}+C_{\Gcal}\}\epsilon_{h} + C_{\Hcal}\epsilon_{g} + 2\sup_{h\in \Hcal,g\in \Gcal }|(\EE_n-\EE)[\prns{Y-h(X)}g(Z) + 0.5 h(X)^2]|}. 
\end{align*}

\end{theorem}

Compared to \Cref{thm:main}, the upper bound in \Cref{thm:misspecification} involves additional misspecification errors $\epsilon_h, \epsilon_g$ due to misspecified $\Hcal, \Gcal$. The empirical process term in \Cref{thm:misspecification} can be again bounded by Rademacher complexities.

\section{Discussions}\label{sec:discussion}

In this section, we compare our method with existing minimax NPIV estimators in \citet{dikkala2020minimax,LiaoLuofeng2020PENE,bennett2022inference} as they are most relevant. 
Other existing minimax estimators are similar so we only briefly review them in  
Section~\ref{subsec:related}.

\subsection{Comparisons to \citet{dikkala2020minimax}}\label{sec:dikkala}

\citet{dikkala2020minimax}  considers the following minimax estimator:
\begin{align*}
    \hat h_{\pro}=\argmin_{h\in \Hcal}\max_{g\in \Gcal}\tilde L_n(h,g),\quad \tilde L_n(h,g)\coloneqq -0.5\EE_n[g^2(Z)] + \EE_n[\prns{Y-h(X)}g(Z)]. 
\end{align*}
Here for simplicity, we omit possible additional regularizers for $h$ and $g$. 
 
\citet{dikkala2020minimax}   assumes the closedness condition that $\Tcal(\Hcal-h^{\diamond}_0)\subset \Gcal$ where  
$h^{\diamond}_0$ can be an arbitrary solution to $\Tcal h = r_0$ (note this condition is invariant to the choice of $h^{\diamond}_0$). Under this condition, letting $L(h,g)$ be the population analog of $\tilde L_n(h,g)$, it can be shown that $\max_{g \in \Gcal}\EE[{\tilde L(h,g)}] = 0.5 \EE[(\Tcal(h^{\diamond}_0-h)[Z])^2] = 0.5 \EE[{\prns{[\Tcal h](Z) - r_0(Z)}^2}]$. 
In other words, the minimax objective in \citet{dikkala2020minimax} is used to approximate the projected MSE objective under the closedness condition. 
In contrast, our proposed minimax objective is motivated by the method of Lagrange multipliers, and it does not need the closedness condition. 

To compare the theory in \citet{dikkala2020minimax} with our theory, we consider finite classes $\Hcal,\Gcal$ for simplicity. 
Then the theory in \citet{dikkala2020minimax} implies that if $\Ncal_{r_0}(\Tcal) \cap \Hcal \neq 0$ and $\Tcal (\Hcal-h^{\diamond}_0)\subset \Gcal$ for $h^{\diamond}_0 \in \Ncal_{r_0}(\Tcal)$, then we have $
 \EE[\{\Tcal(\hat h_{\pro}-h^{\diamond}_0)\}^2(Z)] = O\prns{\prns{\frac{\ln\prns{|\Hcal||\Gcal|/\delta}}{n}}^{1/2}}$ with probability $1-\delta$.

Note that the rate $O((\ln(|\Hcal||\Gcal|)/n)^{1/2})$ above is faster than our rate $O((\ln(|\Hcal||\Gcal|)/n)^{1/4})$ in Corollary~\pref{cor:iv2}. 
However, the rate above is for the weak projected MSE, while  our rate in Corollary~\pref{cor:iv2} is for the stronger $L_2$ error, so they are not  comparable. 
In particular, the projected MSE rate cannot translate into an $L_2$ rate without further restrictions. 
\citet{dikkala2020minimax} consider restrictiting the ill-posedness measure $\sup_{h \in \Hcal}\frac{ \EE[\{\hat h-h)\}^2(X)] }{\EE[\{\Tcal(\hat h-h)\}^2(Z)]}$. However, this ill-posedness measure may generally be infinite, and in fact is guaranteed to be infinite when the solutions to the NPIV problem are nonunique, so using it to get $L_2$ convergence rates is often problematic.

\begin{remark}[Enjoy the best of both worlds]
Here we observe that the estimator in \citet{dikkala2020minimax} can achieve a fast projected MSE rate while our estimator achieves a slow $L_2$ rate. 
One may wonder whether it is possible to achieve both guarantees at the same time. 
We explore this question in \pref{sec:both_world} and find this is possible if we put aside computational considerations. 
\end{remark}

\subsection{Comparison to \citet{LiaoLuofeng2020PENE}}\label{sec:liao}

\citet{LiaoLuofeng2020PENE}
builds on \citet{dikkala2020minimax} and incorporates additional 
Tikhonov regularization into the minimax optimization: 
\begin{align}\label{eq: Liao}
\min_{h\in \Hcal}\max_{g\in \Gcal}\,  -0.5\EE_n[g^2(Z)] + \EE_n[\prns{Y-h(X)}g(Z)] + \alpha \EE_n[h^2(X)].
\end{align}
\citet{LiaoLuofeng2020PENE}  also needs the closedness assumption in \citet{dikkala2020minimax} and a realizability assumption that the Tikhonov regularized solution $h_{0, \alpha}$ is contained in $\Hcal$ for small $\alpha$.
In addition, they assume 
that the NPIV solution is unique and satisfies a source condition with exponent $\beta \in (0,1]$, and the regularization strength $\alpha$ vanishes to $0$ at an appropriate rate as $n \to \infty$. Under these conditions, they can derive an $L_2$ convergence rate. In particular, their $L_2$ rate has the order $O(n^{-1/6})$ when the function classes are \emph{e.g.} finite or VC, and $\beta = 1$. 

Our proposed estimator and theory significantly differ from \citet{LiaoLuofeng2020PENE}. 
Specifically, our estimator does not involve the $\EE_n[g^2(Z)]$ term and our regularized term $\EE_n[h^2(X)]$ has a constant coefficient $0.5$ but \Cref{eq: Liao} needs a vanishing $\alpha$. 
Moreover, our theory accommodates non-unique solutions,
and uses  different realizability assumptions. 
Notably, under our source condition $\beta = 1$, our convergence rate $O(n^{-1/4})$ is faster than the rate $O(n^{-1/6})$ in \citet{LiaoLuofeng2020PENE}.

\subsection{Comparison to \citet{bennett2022inference} }\label{sec: compare-bennett}

Under the same source condition, \citet{bennett2022inference} \footnote{{ Note the main focus of \citet{bennett2022inference} is to estimate the Riesz representator (in their notation, $q^{\dagger}$) with $L_2$ error rates. However, their argument is easily adapted to our scenario.  } } formulate the $L_2$ error of $h_0$ as projected MSEs: $
    \EE[\{h_0-h\}^2(X) ] = \EE[(\Tcal^{*}\{\bar g_0-g\})^2(X)]$
where $\Tcal^{*} g=h$ and $\Tcal^{*} \bar g_0=h_0$. First, note that for any fixed $\bar g_0$ such that $T^\star \bar g_0 = h_0$, we have
\begin{align*}
   \Ncal_{h_0}(\Tcal^{*}) =  \argmin_{g \in \Gcal}\EE[(\Tcal^{*}\{\bar g_0-g\})^2(X)]= \argmin_{g \in \Gcal} 0.5\EE[(\Tcal^{*}g)^2(X)]- \EE[Yg(Z)]. 
\end{align*}
Then, under the closedness assumption $\Tcal^{*}\Gcal \subset \Hcal$, we have
\begin{align*}
    \Ncal_{h_0}(\Tcal^{*}) = \argmax_{g \in \Gcal}\min_{h \in \Hcal}0.5\EE[h^2(X)]+\EE[\{Y-h(X)\}g(Z)].
\end{align*}
Then, noting that the inner minimizer $h$ satisfies $\Tcal^{*}g=h$ for any given $g$, and recalling the original goal is to find $h_0$ such that $\Tcal^{*}\bar g_0=h_0$, we can deduce that \footnote{Here, letting a loss function to be $L^{\star}(h,g)$, the equation $(\bar g_0,h_0)=\argmin_{g}\argmax_{h}L(h,g)$ means $\bar g_0=\argmin_{g}\max_{h}L(h,g)$ and $h_0=\argmax_{h}L(h,\bar g_0)$.   }
\begin{align*}
    \{\bar g_0, h_0\} = \argmax_{g \in \Gcal}\argmin_{h \in \Hcal}0.5\EE[h^2(X)]+\EE[\{Y-h(X)\}g(Z)]. 
\end{align*}
Finally, their proposed estimator $\hat h_{\text{fli}}$ is given by replacing expectations with empirical averages.

In comparison to our proposed estimator $\hat h_{\mn}$, the difference lies in the flip of $\argmax$ and $\argmin$. Since $\Gcal,\Hcal$ could be non-convex, the two estimators are generally different. Indeed, this results in a significant difference in terms of the required assumptions. In $\hat h_{\text{fli}}$, the primary assumptions are the source condition, $\bar g_0 \in \Gcal$, and $\Tcal^{*}\Gcal \subset \Hcal$ (note that $h_0 \in \Hcal$ is implicit from the latter two conditions). Conversely, in our proposed estimator $\hat h_{\mn}$, the primary assumptions are the source condition and $\bar g_0 \in \Gcal,h_0 \in \Hcal$. This condition is strictly weaker as we dispense with the requirement of closedness. This improvement is significant due to the inherent conflict between the source condition and closedness,
as elucidated next.

\subsection{Tension between Source Condition and Closedness}
\label{subsec:tension}

In \Cref{sec:dikkala,sec:liao,sec: compare-bennett}, the existing estimators all require certain closedness assumption, either $\Tcal(\Hcal-h^{\diamond}_0)\subset \Gcal$ for an arbitrary   solution 
$h^{\diamond}_0$ to $\Tcal h = r_0$, or $\Tcal^{*}\Gcal \subset \Hcal$.
In contrast, our proposed estimator does not need any closedness assumption. 
In this subsection, we show that the closedness conditions are  inherently in tension with the source condition.    
This illustrates the benefit of getting rid of the source condition. 
For simplicity, we consider  a compact linear operator $\Tcal$ that admits a singular value decomposition (SVD) $\{\sigma_i,u_i,v_i\}_{i=1}^{\infty}$, where $\{u_i\}_{i=1}^{\infty},\{v_i\}_{i=1}^{\infty}$ are orthonormal bases in the Hilbert spaces $L_2(Z),L_2(X)$, respectively, and $\sigma_1\geq \sigma_2\geq \cdots$ are the singular values.
It follows that $\Tcal v_i =\sigma_i u_i,\Tcal^* u_i =\sigma_i v_i $, and $\Tcal\Tcal^{\star}$ has the SVD $\{\sigma^2_i,u_i,u_i\}_{i=1}^{\infty}$.
Here we assume a compact operator merely for a simple countable SVD. Non-compact operators can be handled similarly, but involve more cumbersome notations \citep{cavalier2011inverse}. 

To understand the source condition in Assumption~\ref{assum:source}, we write the function $r_0$ as $r_0 = \sum_{i=1}^{\infty} \gamma_i u_i$ with $\sum_{i = 1}^{\infty} \gamma^2_i<\infty$.
The source condition $r_0 \in \Rcal(\Tcal\Tcal^{\star})$ means that there exists $\bar g_0=\sum_{i=1}^{\infty}\beta_i u_i $ with $\sum_{i=1}^{\infty}\beta^2_i<\infty$ such that $h_0= \Tcal^{*}\Tcal\bar g_0$. It follows from the SVD of $\Tcal\Tcal^{\star}$ that $\gamma_i=\sigma^2_i\beta_i$. Therefore, the source condition requires $\sum_{i=1}^{\infty} \gamma_i^2/\sigma^4_i<\infty$. This means that the function $r_0$ needs to be sufficiently smooth relative to the spectrum of $\mathcal{T}$. 
Obviously, the source condition is more readily satisfied when the decaying rate of $\{\sigma_i\}_{i=1}^{\infty}$ is slower, \ie, when the operators $\Tcal$ and $\Tcal^\star$ are less smooth. 
In contrast, the closedness conditions are generally more easily satisfied when  $\{\sigma_i\}_{i=1}^{\infty}$ decays faster and the operators $\Tcal$ and $\Tcal^\star$ are more smooth. 

Hence, we observe that the source condition and closedness imply opposing restrictions on the smoothness of the operators $\Tcal$ and $\Tcal^\star$.

\section{Conclusion}\label{sec:conclusion}

In this paper, we study NPIV regression with general function approximation.
We propose a penalized minimax estimator based on a novel constrained optimization formulation of the least norm IV solution. 
We prove that our estimator converges to this least norm solution, and derive its 
$L_2$ convergence rate under a source condition and realizability assumptions on both function classes for the minimax estimator. 
Notably, our estimator does not require uniqueness of the NPIV solution, and it avoids a closedness condition commonly assumed for existing minimax estimators.
There are many interesting future directions of research. One direction is extending our work to more general inverse problems, including nonlinear inverse problems \citep{ito2014inverse}. Another direction is extending our work to IV quantile regression \citep{chernozhukov2017instrumental}.

\bibliographystyle{chicago}
\bibliography{ref_COLT}

\newpage
\appendix

\section{Enjoy the Best of Both Worlds } \label{sec:both_world}

Thus far, we have encountered two types of guarantees: slow $L_2$ rates and fast projected MSEs. The next step is to obtain guarantees that possess both properties. If we put aside issues of computational efficiency, then this is actually achievable. The estimator is defined as follows: 
\begin{align*}
    \hat h_{\both}=\argmin_{h\in \Hcal_n}\max_{g\in \Gcal}\tilde L_n(h,g)
\end{align*}
where  
\begin{align*}
    \Hcal_n =\{h\in \Hcal; \max_{g\in \Gcal}L_n(h,g) - \min_{h\in \Hcal}\max_{g\in \Gcal}L_n(h,g)\leq \mu_n 
    \}. 
\end{align*}
Here, $\mu_n$ is some hyperparameter. The set $\Hcal_n$ is defined so that each of its element has the $L_2$ convergence guarantee under the source condition.    

\begin{theorem}[fast projected MSEs + slow $L_2$ errors]
\label{thm:both}
Suppose $\Hcal,\Gcal$ are finite for simplicity. Suppose $h_0 \in \Hcal , \Tcal(\Hcal-h_0)\subset \Gcal,  \Ncal_{h_0}(\Tcal^{\star})\cap \Gcal \neq \emptyset$.  
Then, when we take $\mu_n = (C_{\Hcal}+C_{\Gcal})^2\sqrt{\ln(|\Hcal||\Gcal| /\delta)/n}$, with probability $1-\delta$, we have

\begin{align*}
    \|\Tcal(\hat h_{\both}-h_0)\|_2 \leq c(C_{\Hcal}+C_{\Gcal})\sqrt{\frac{\ln(|\Hcal||\Gcal|/\delta)}{n}},\quad \| \hat h_{\both}-h_0 \|_2 \leq c(C_{\Hcal}+C_{\Gcal}) \prns{\frac{\ln(|\Hcal||\Gcal|/\delta)}{n}}^{1/4}. 
\end{align*}
\end{theorem}

\section{General Characterization of Saddle Points}\label{sec:general}

First, notice 
\begin{align}\label{eq:best_response}
    \{h_0\} = \min_{h \in L_2(X)}L(h,\bar g_0)
\end{align}
for any $\bar g_0 \in \Ncal_{h_0}(\Tcal^{*})$. In other words, the optimal response to $L(h,\bar g_0)$ is uniquely $h_0$.  
It follows from two observations: (1) $h_0$ is a best response for any element $\bar g_0$ in $\Ncal_{h_0}(\Tcal^{*})$, since $(h_0,\bar g_0)$ is a saddle point by \pref{lem:saddle_point}; and (2) the best response for each $\bar g_0$ is unique, since $L(h,\bar g_0)$ is strictly convex in $h$, due to the $\langle h,h\rangle_{L_2(X)}$ term.

Next, we invoke the following general characterization of saddle points. Here, $(\tilde x,\tilde y) \in \argmin_{x \in \Xcal'}\argmax_{y'\in \Ycal'}f(x,y)$ means $\tilde x \in \argmin_{x \in \Xcal'}\max_{y\in \Ycal'}f(x,y)$ and $\tilde y \in \argmax_{y\in \Ycal'}f(\tilde x,y)$.

\begin{lemma}[Characterization of saddle points over constrained sets]\label{lem:invaraince_lemma}
Let $\Zcal$ be a set of saddle points for $f(x,y)$ over $\Xcal,\Ycal$. Let $\Zcal_{\Xcal}=\argmin_{x \in \Xcal}\max_{y\in \Ycal}f(x,y), 
( \cdot,\tilde \Zcal_{\Xcal})=\argmax_{y\in \Ycal}\argmin_{x \in \Xcal}f(x,y).$
Then, for  $\Xcal'\subset \Xcal,\Ycal'\subset \Ycal$, if $\Zcal \cap (\Xcal',\Ycal')$ is non-empty, we have 
\begin{align}\label{eq:one_direction}
   \Zcal_{\Xcal} \cap \Xcal' \subset \argmin_{x \in \Xcal'}\max_{y\in \Ycal'}f(x,y)
\end{align}
and 
\begin{align}\label{eq:second_direction}
  \argmin_{x \in \Xcal'}\max_{y\in \Ycal'}f(x,y) \subset \tilde \Zcal_{\Xcal}\cap \Xcal'.
\end{align}
\end{lemma}
In \pref{lem:invaraince_lemma}, the primary assumption $\Zcal \cap (\Xcal',\Ycal')\neq \emptyset$ means that some saddle point (with respect to $\Xcal,\Ycal$) is included in $\Xcal',\Ycal'$. The equation \eqref{eq:one_direction} states that any saddle points ($ \Zcal_{\Xcal} \cap \Xcal')$ over unconstrained function classes ($\Xcal,\Ycal$) are still saddle points over constrained function classes ($\Xcal',\Ycal'$). The equation \eqref{eq:second_direction} states that any saddle point over constrained function classes ($\Xcal',\Ycal'$) is included in $\tilde \Zcal_{\Xcal}$. 

We combine the above characterization of saddle points with \pref{lem:saddle_point} by setting $(\Xcal',\Ycal')=(\Hcal,\Gcal),(\Xcal,\Ycal)=(L_2(X),L_2(Z)),f(x,y)=L(h,g)$. As an immediate consequence, when $h_0 \in \Hcal, \Ncal_{h_0}(\Tcal^{\star})\cap \Gcal \neq \emptyset $ (i.e., saddle points are included in $(\Hcal,\Gcal)$), using \pref{eq:one_direction}, we have 
$   \{h_0\} \subset \argmin_{h \in \Hcal}\max_{g \in \Gcal}L(h,g).$
Next, using \pref{eq:second_direction}, we have 
$
 \argmin_{h \in \Hcal}\max_{g \in \Gcal}L(h,g) \subset \{h_0\}$
since $(\cdot,h_0)=\argmax_{g \in \Gcal}\argmin_{h \in \Hcal}L(h,g)$ by \eqref{eq:best_response}.

\section{Computational Perspective}\label{sec:comptuational}

To solve the optimization problem in \Cref{eq:motivatioin3}, we can leverage the recent advances in minimax optimization algorithms, even when the function classes $\Hcal$ and $\Gcal$ are neither convex nor concave, such as neural network classes \citep{daskalakis2017training}.
In particular,  using a Reproducing kernel Hilbert space (RKHS) ball as $\Gcal$ is particularly convenient, since then the inner maximization problem in \Cref{eq:motivatioin3} has a closed form solution. 
Specifically, when 
 $\Gcal=\{g:\|g\|_{K}\leq 1\}$ for a positive definite kernel $K:D_Z\times D_Z \to \RR$ and its associated RKHS norm  $\|\cdot\|_{K}$, \Cref{eq:motivatioin3} reduces to 
\begin{align*}
   \argmin_{h\in \Hcal} 0.5 \EE_n[h^2(X)] + \prns{\frac{1}{n^2}\sum_{i=1}^n\sum_{j=1}^n \prns{Y_i-h(X_i)}K(Z_i,Z_j)\prns{Y_j-h(X_j)}}^{1/2}. 
\end{align*}

\section{Proof in \pref{sec:setup} }

\subsection{Proof of \pref{lem:minimum}}

Here, we have $\Ncal_{r_0}(\Tcal)=h_0 + \Ncal(\Tcal)$.
The least norm solution $h_0$ among $\Ncal_{r_0}(\Tcal)$ is the projection of any element in $\Ncal_{r_0}(\Tcal)$ onto (the closed subspace) $\Ncal(\Tcal)^\perp$. Hence,  $\{h_0\} =\Ncal(\Tcal)^{\perp}\cap \Ncal_{r_0}(\Tcal)=\overline{\Rcal(\Tcal^{\star})} \cap \Ncal_{r_0}(\Tcal) $.  Here, we use $\Ncal(\Tcal)^{\perp}=\overline{\Rcal(\Tcal^{\star})} $.

\section{Proof in \pref{sec:primary_identification}}

\subsection{Proof of \pref{lem:norm}}

It is clear from \pref{lem:minimum}. 

\subsection{Proof of \pref{lem:saddle_point}}

The proof is as follows. From \pref{sec:auxiliary}, a point $(h',g')\in (L_2(X),L_2(Z))$ is a saddle point if and only if the strong duality holds and  $h' \in    \argmin_{h \in L_2(X)}\sup_{g \in L_2(Z)}L(h,g)$ and $g' \in \argmax_{g \in L_2(Z)}\inf_{h \in L_2(X)}L(h,g).$ We check this condition. 

Hence, we first show 
\begin{align}\label{eq:duality1}
    \{h_0\} = \argmin_{h \in L_2(X)}\sup_{g \in L_2(Z)}L(h,g),  \quad 0.5 \|h_0\|^2_2= \min_{h \in L_2(X)}\sup_{g \in L_2(Z)}L(h,g)
\end{align}
First, for any $h\neq \Ncal_{r_0}(\Tcal)$, we have $\sup_{g \in L_2(Z)}L(h,g)=\infty$. Hence, the solution needs to belong to $\Ncal_{r_0}(\Tcal)$. Since $\sup_{g \in L_2(Z)}L(h,g)=0.5 \EE[h^2(X)]$ for any $h\in \Ncal_{r_0}(\Tcal)$, using \pref{lem:norm}, thus, from the definition of $h_0$, the solution is $h_0$.

Next, we show  
\begin{align}\label{eq:duality2}
 \Ncal_{h_0}(\Tcal^{\star})= \argmax_{g \in L_2(Z)}\inf_{h \in L_2(X)}L(h,g),  \quad 0.5\|h_0\|^2_2= \max_{g \in L_2(Z)}\inf_{h \in L_2(X)}L(h,g). 
\end{align}
We solve the inner minimization problem first. Then, 
\begin{align*}
   \inf_{h \in L_2(X)}L(h,g) &= \inf_{h \in L_2(X)} 0.5\|h -\Tcal^{\star} g \|^2_2 + \langle r_0,g\rangle_{L_2(Z)}- 0.5\langle \Tcal^{\star} g,\Tcal^{\star} g\rangle_{L_2(X)}\\
   &=\langle r_0,g\rangle_{L_2(Z)}- 0.5 \langle \Tcal^{\star} g,\Tcal^{\star} g\rangle_{L_2(X)} \\
   &= \langle \Tcal h_0,g\rangle_{L_2(Z)}- 0.5 \langle \Tcal^{\star} g,\Tcal^{\star} g\rangle_{L_2(X)} \tag{Use $r_0 = \Tcal h_0$} \\
   &= -0.5\|\Tcal^{\star}g - h_0\|^2_2 + 0.5 \|h_0\|^2_2. 
\end{align*}
By using Assumption~\ref{assum:source}, since $\Ncal_{h_0}(\Tcal^{*})$ is not empty, we have 
\begin{align*}
    \Ncal_{h_0}(\Tcal^{*}) = \argmax_{g \in L_2(Z)}\inf_{h \in L_2(X)}L(h,g).  
\end{align*}
Finally, since the strong duality holds from \pref{eq:duality1} and \pref{eq:duality2}, the set of saddle points is $(h_0,\Ncal_{h_0}(\Tcal^{\star}))$.

\subsection{Proof of \pref{lem:saddle_point2}}

Recall the saddle point exists if and only if 
 $\argmin_{h \in L_2(X)}\sup_{g \in L_2(Z)}L(h,g)$ and $ \argmax_{g \in L_2(Z)}\inf_{h \in L_2(X)}L(h,g)$ exist and the strong duality holds. We already show that Assumption~\ref{assum:source} is sufficient to ensure the existence of the saddle point. In this proof, we show Assumption~\ref{assum:source} is necessary to ensure the existence of the saddle point. 
 
To ensure the existence of saddle point, we need to ensure the existence of $ \argmax_{g \in L_2(Z)}\inf_{h \in L_2(X)}L(h,g)$.  This optimization problem is equivalent to 
 \begin{align}\label{eq:projection}
     \argmin_{g\in L_2(Z)}\|\Tcal^{\star}g-h_0\|^2_2
 \end{align}
as we see in the proof of \pref{lem:saddle_point}. 
This solution exists if and only if $h_0 \in \Rcal(\Tcal^{\star})+\Rcal(\Tcal^{\star})^{\perp}$.
To prove this, we define a projection operator onto $\overline{\Rcal(\Tcal^{\star})}$ as $P_{\overline{\Rcal(\Tcal^{\star})}}$. Then, the solution of \eqref{eq:projection} exists if and only if  $ P_{\overline{\Rcal(\Tcal^{\star})}}h_0 \in \Rcal(\Tcal^{\star})$. Here, $ P_{\overline{\Rcal(\Tcal^{\star})}}h_0 \in \Rcal(\Tcal^{\star})$ implies 
\begin{align*}
     h_0 = P_{\overline{\Rcal(\Tcal^{\star})}}h_0  + (I-P_{\overline{\Rcal(\Tcal^{\star})}})h_0 \in  \Rcal(\Tcal^{\star})+ \Rcal(\Tcal^{\star})^{\perp}. 
\end{align*}
Besides, $h_0 \in \Rcal(\Tcal^{\star})+\Rcal(\Tcal^{\star})^{\perp}$ implies $P_{\overline{\Rcal(\Tcal^{\star})}}h_0 \in \Rcal(\Tcal^{\star})$ recalling $ \Hcal= \overline{\Rcal(\Tcal^{\star})} \bigoplus \Rcal(\Tcal^{\star})^{\perp}$. 
This finishes proving that the  solution of \eqref{eq:projection} exists if and only if $h_0 \in \Rcal(\Tcal^{\star})+\Rcal(\Tcal^{\star})^{\perp}$.

Finally, recall $h_0 \in \overline{\Rcal(\Tcal^{\star})}  $ using \pref{lem:minimum}. Thus, $h_0 \in \Rcal(\Tcal^{\star})+\Rcal(\Tcal^{\star})^{\perp}$  
implies 
 $h_0 \in \Rcal(\Tcal^{\star})$ since if $h_0 =h_{0,2} + h_{0,3}, h_{0,2} \in \Rcal(\Tcal^{\star}),h_{0,3} \in \Rcal(\Tcal^{\star})^{\perp}$, we have $h_{0,3}=h_0 - h_{0,2}\in \overline{\Rcal(\Tcal^{\star})}\cap \Rcal(\Tcal^{\star})^{\perp}=\{0\}$. 
 
 The statement is concluded by the fact $h_0 \in \Rcal(\Tcal^{\star})$ implies $r_0 \in \Rcal(\Tcal \Tcal^{\star})$.

\subsection{Proof of \pref{lem:invaraince_lemma}}

Clearly, each element in $\Zcal \cap (\Xcal,\Ycal)$ is a saddle point over $\Xcal',\Ycal'$ since this is a saddle point over $\Xcal,\Ycal$. Therefore, 
\begin{align*}
    \Zcal_{\Xcal} \cap \Xcal' \subset \argmin_{x\in \Xcal'}\max_{y \in \Ycal'}f(x,y). 
\end{align*}

Now, we prove the second statement. Let $(x_0,y_0)$ be an element in $\Zcal \cap (\Xcal,\Ycal)$ (this exists and this is a saddle point).
Then, take:
\begin{align*}
    \tilde x \in \argmin_{x \in \Xcal'}\sup_{y\in \Ycal'}f(x,y), \quad  \tilde y \in \argmax_{y\in \Ycal'}\inf_{x \in \Xcal'}f(x,y). 
\end{align*}
Since $(\tilde x,\tilde y)$ is a saddle point over $\Xcal',\Ycal'$, we have 
\begin{align*}
 f(x_0, y_0)\geq f(x_0,\tilde y)  \geq f(\tilde x,\tilde y)\geq f(\tilde x, y_0)\geq  f(x_0, y_0).
\end{align*}
Then, the above inequalities are equalities. Hence, we have 
\begin{align*}
    f(x_0, \tilde y) = f(\tilde x,y_0), \quad    f(x_0, y_0)=f(x_0, \tilde y)
\end{align*}
This means that 
\begin{align*}
    \tilde x \in \Zcal'_{\Xcal} \subset \Xcal'. 
\end{align*}
recalling $y_0 \in \argmax_{y\in \Ycal}\min_{x\in  \Xcal}f(x,y)$.

\subsection{Proof of \pref{thm:idenfication}}

We show two proofs. 

\paragraph{First Proof.}

We use \pref{lem:invaraince_lemma}. First, using \eqref{eq:one_direction}, 
\begin{align*}
    \{h_0\} \subset \argmin_{h \in \Hcal}\max_{g\in \Gcal}L(h,g).
\end{align*}
Second, we use \eqref{eq:second_direction}. Here, recalling the proof of \pref{lem:saddle_point}, we have 
\begin{align*}
   (\Ncal_{h_0}(\Tcal^{*}),h_0)) \in  \argmax_{g\in \Gcal}\argmin_{h \in \Hcal}L(h,g). 
\end{align*}
Therefore, using \pref{lem:saddle_point}, we have 
\begin{align*}
    \argmin_{h \in \Hcal}\max_{g\in \Gcal}L(h,g) \subset \{h_0\}. 
\end{align*}
Hence, 
\begin{align*}
      \argmin_{h \in \Hcal}\max_{g\in \Gcal}L(h,g) = \{h_0\}. 
\end{align*}

\paragraph{Second Proof.}

We give more direct proof to show the finite sample result later. 

We take some element $\bar g_0$ from $\Ncal_{h_0}(\Tcal^{\star}) \cap \Gcal$. This satisfies $\Tcal^{*} \bar g_0=h_0$. 
We define
\begin{align*}
    L(h,g) &\coloneqq  0.5\EE[h^2(X)]+\EE[g(Z)\{Y-h(X)\}],\\
    \hat g(h) &\coloneqq \argmax_{g \in \Gcal}      L(h,g),\quad 
    \hat h \coloneqq \argmin_{h \in \Hcal}\sup_{g \in \Gcal}L(h,g),
\end{align*}
and $\hat g \coloneqq \hat g(\hat h)$. Hence, for any $h\in \Hcal$, 
\begin{align*}
 &  L(h,\bar g_0)  -   L(h_0,\bar g_0) \\
  &= 0.5 \EE[h^2(X)]+\EE[\bar g_0(Z)\{Y-h(X)\}]-0.5 \EE[\{h_0\}^2(X)] 0 \EE[\bar g_0(Z)\{Y-h_0(X)\}]\\
    &= 0.5\EE[h^2(X)]+\EE[\bar g_0(Z)\{h_0(X)-h(X)\}]- 0.5\EE[\{h_0\}^2(X)]\\
    &= 0.5\EE[h^2(X)]+\EE[h_0(X)\{h_0(X)-h(X)\}]- 0.5\EE[\{h_0\}^2(X)] \tag{We use $\EE[h_0(X)\mid Z]=\bar g_0(Z)$}\\
    &= 0.5\EE[\{h(X)-h_0(X)\}^2]. 
\end{align*}
Therefore, for any $h\in \Hcal$, 
\begin{align}\label{eq:key}
\EE[\{h(X)-h_0(X)\}^2] =  L(h,\bar g_0)-L(h_0,\bar g_0) . 
\end{align}
Furthermore, 

\begin{align*}
    L(\hat h,\hat g) &\geq L(\hat h,\bar g_0) \tag{Construction of estimators} \\ 
    &\geq L(h_0,\bar g_0)  \tag{Saddle point property} \\ 
     &\geq L(h_0,\hat g(h_0)). \tag{Saddle point property}
\end{align*}
Since we have $L(\hat h,\hat g)\leq L(h_0,\hat g(h_0))$ from the definition, all of the above inequalities are equalities. Then, we have
\begin{align}\label{eq:upper_bound}
     L(\hat h,\bar g_0)-L(h_0,\bar g_0) =0. 
\end{align}
In conclusion, combining \eqref{eq:key} with \eqref{eq:upper_bound}, we have 
\begin{align*}
\EE[\{\hat h(X)-h_0(X)\}^2]\leq L(\hat h,\bar g_0)-L(h_0,\bar g_0) =0.  
\end{align*}
Hence, $\hat h(X)=h_0(X)$.

\section{Proof of \pref{sec:finite}}

\subsection{Proof of \pref{thm:main}}

We take some element $\bar g_0$ from $\Ncal_{h_0}(\Tcal^{\star}) \cap \Gcal$. This satisfies $\Tcal^{*} \bar g_0=h_0$. 
We define
\begin{align*}
    L(h,g) &\coloneqq  0.5\EE[h^2(X)]+\EE[g(Z)\{Y-h(X)\}], \\
    L_n(h,g) &\coloneqq  0.5\EE_n[h^2(X)]+\EE_n[g(Z)\{Y-h(X)\}], \\
    \hat g(h) &\coloneqq  \argmax_{g \in \Gcal}  0.5\EE_n[h^2(X)]+\EE_n[g(Z)\{Y-h(X)\}], \\
    M(\Hcal,\Gcal) &\coloneqq  \sup_{h\in \Hcal,g\in \Gcal }|(\EE_n-E)[\{Y-h(X)\}g(Z) + 0.5 h(X)^2]|. 
\end{align*}
Using \eqref{eq:key}, recall for any $h \in \Hcal$, we have 
\begin{align*}
\EE[\{h(X)-h_0(X)\}^2] =  L(h,\bar g_0)-L(h_0,\bar g_0) . 
\end{align*}
Here, we have 
\begin{align*}
    L_n(\hat h,\hat g(\hat h)) &\geq L_n(\hat h,\bar g_0) \tag{Construction of estimators} \\ 
    &\geq L(\hat h,\bar g_0) -  M(\Hcal,\Gcal) \\ 
    &\geq L(h_0,\bar g_0)  -  M(\Hcal,\Gcal) \tag{Saddle point property} \\ 
     &\geq L(h_0,\hat g(h_0)) -  M(\Hcal,\Gcal) \tag{Saddle point property}  \\
     &\geq L_n(h_0,\hat g(h_0)) -  2M(\Hcal,\Gcal) \\ 
    &\geq  L_n(\hat h,\hat g(\hat h) -  2M(\Hcal,\Gcal).   \tag{Construction of estimators.}
\end{align*}
Therefore, we have 
\begin{align*}
    L(\hat h,\bar g_0)-L(h_0,\bar g_0)\leq 2M(\Hcal,\Gcal).
\end{align*}
Finally, we have 
\begin{align*}
    \EE[\{\hat h(X)-h_0(X)\}^2]\leq L(\hat h,\bar g_0)-L(h_0,\bar g_0) \leq{2M(\Hcal,\Gcal)}. 
\end{align*}

\subsection{Proof of Corollary~\ref{cor:rademacher}}

We calculate the following empirical process term: 
\begin{align*}
   \sup_{h\in \Hcal,g\in \Gcal }|(\EE_n-\EE)[\{Y-h(X)\}g(Z) + 0.5 h(X)^2]|. 
\end{align*}    
Then, from \citet[Theorem 4.10]{wainwright2019high}, this is upper-bounded by 
\begin{align*}
    c \braces{ \mathfrak{R}_n(\Acal_1) +   \mathfrak{R}_n(\Acal_2) +  \mathfrak{R}_n(\Acal_3)+(C_{\Gcal}+C_{\Hcal})C_{\Hcal}\sqrt{\ln(1/\delta)/n} }  
\end{align*}
where 
\begin{align*}
    \Acal_1 = \{ y g(z) ; g \in \Gcal\}, \quad 
    \Acal_2 = \{ h(x)g(z); h \in \Hcal, g\in \Gcal\}, \quad 
    \Acal_3 =  \{ 0.5 h(x)^2; h \in \Hcal\}.  
\end{align*}
First, we have 
\begin{align*}
    \mathfrak{R}_n(\Acal_1)\lesssim C_{\Gcal}    \mathfrak{R}_n(\Gcal). 
\end{align*}
Secondly, we have 
\begin{align*}
    \mathfrak{R}_n(\Acal_2)\lesssim (C_{\Hcal} + C_{\Gcal})(\mathfrak{R}_n(\Gcal) +\mathfrak{R}_n(\Hcal)). 
\end{align*}
Here, we use the proof of \citet[Proof of Corollary 3]{kallus2021causal}. Thirdly, we have 
\begin{align*}
     \mathfrak{R}_n(\Acal_3)\lesssim 2C_{\Hcal}\mathfrak{R}_n(\Hcal). 
\end{align*}
Combining all results together, the empirical process term is upper-bounded by 
\begin{align*}
    c \braces{(C_{\Hcal} + C_{\Gcal})(\mathfrak{R}_n(\Gcal) +\mathfrak{R}_n(\Hcal))+ (C_{\Gcal}+C_{\Hcal})C_{\Hcal}\sqrt{\ln(1/\delta)/n} } . 
\end{align*}

\subsection{Proof of Corollary~\ref{cor:iv}}

We combine the Dudley integral \pref{thm:dudley} with Corollary~\ref{cor:rademacher}.

\subsection{Proof of \pref{thm:misspecification}  }

We take some element $\bar g_0$ from $\Ncal_{h_0}(\Tcal^{\star}) \cap \Gcal$. This satisfies $\Tcal^{*} \bar g_0=h_0$. 
\begin{align*}
    L(h,g) &\coloneqq  0.5\EE[h^2(X)]+\EE[g(Z)\{Y-h(X)\}], \\
    L_n(h,g) &\coloneqq  0.5\EE_n[h^2(X)]+\EE_n[g(Z)\{Y-h(X)\}], \\
    \hat g(h) &\coloneqq  \argmax_{g \in \Gcal}  0.5\EE_n[h^2(X)]+\EE_n[g(Z)\{Y-h(X)\}], \\
    M(\Hcal,\Gcal) &\coloneqq  \sup_{h\in \Hcal,g\in \Gcal }|(\EE_n-E)[\{Y-h(X)\}g(Z) + 0.5 h(X)^2]|. 
\end{align*}
and $\hat g=\hat g(\hat h)$. Recall for any $h\in \Hcal$, 
\begin{align*}
\EE[\{h(X)-h_0(X)\}^2]\leq L(h,\bar g_0)-L(h_0,\bar g_0) . 
\end{align*}
Furthermore, 
\begin{align*}
    & L(\hat h, \bar g_0 )-L(h_0,\bar g_0) \\
    &=   \underbrace{-L(h_0,\bar g_0) + L(h^{\dagger}, \hat g(h^{\dagger}) )}_{(a)} 
    \underbrace{- L(h^{\dagger}, \hat g(h^{\dagger}) )+L(\hat h,  g^{\dagger} )}_{(c)}\underbrace{-L(\hat h,  g^{\dagger} ) + L(\hat h, \bar g_0 )}_{(f)}.   
\end{align*}
Term (a) is upper-bounded as follows: 
\begin{align*}
    L(h^{\dagger}, \hat g(h^{\dagger}) )-L(h_0,\bar g_0) &\leq 0.5\EE[\{h^{\dagger}\}^2(X)] + \|\hat g(h^{\dagger})\|_2 \|h_0-h^{\dagger}\|_2- 0.5\EE[h^2_0(X)] \\
    &\leq 0.5\EE[\{h^{\dagger}\}^2(X)] + \sup_{g}\|g\|_2 \|h_0-h^{\dagger}\|_2- 0.5\EE[h^2_0(X)] \\
    & \leq  0.5 \|h^{\dagger}+h_0\|_2  \|h^{\dagger}- h_0\|_2 + \{\sup_{g}\|g\|_2\}\|h_0-h^{\dagger}\|_2\\
    &\leq \{2C_{\Hcal}+C_{\Gcal}\}\|h_0-h^{\dagger}\|_2.
\end{align*}
Term (c) is upper-bounded as follows: 
\begin{align*}
    - L(h^{\dagger}, \hat g(h^{\dagger}) )+L(\hat h, g^{\dagger} ) & \leq - L(h^{\dagger}, \hat g(h^{\dagger}) )+ L_n(h^{\dagger}, \hat g(h^{\dagger}) )-L_n(h^{\dagger}, \hat g(h^{\dagger}) )+L_n(\hat h, g^{\dagger} )-L_n(\hat h, g^{\dagger} )+L(\hat h, g^{\dagger} )\\
    &\leq M(\Hcal,\Gcal) +(-L_n (h^{\dagger},\hat g(h^{\dagger}) ) + L_n(\hat h,\hat g)) +M(\Hcal,\Gcal) \\ 
    &\leq 2M(\Hcal,\Gcal). 
\end{align*}
The term (f) is upper-bounded as follows: 
\begin{align*}
    L(\hat h, \bar g_0 )-L(\hat h,  g^{\dagger} )  \leq \|\hat h\|_2\|g_0 - g^{\dagger}\|_2 \leq C_{\Hcal}\|g_0 - g^{\dagger}\|_2 
\end{align*}
In conclusion, we have 
\begin{align*}
\EE[\{\hat h(X)-h_0(X)\}^2]\leq L(\hat h,\bar g_0)-L(h_0,\bar g_0) \leq \{2C_{\Hcal}+C_{\Gcal}\}\|h_0-h^{\dagger}\|_2+ C_{\Hcal}\|g_0 - g^{\dagger}\|_2 + 2M(\Hcal,\Gcal). 
\end{align*}

\section{Proof of \pref{sec:discussion}}

\subsection{Proof of Rate in \pref{sec:dikkala}}

Recall 
\begin{align*}
    \hat h_{\pro}=\argmin_{h\in \Hcal}\max_{g\in \Gcal}\tilde L_n(h,g),\quad \tilde L_n(h,g)\coloneqq -0.5\EE_n[g^2(Z)] + \EE_n[\{Y-h(X)\}g(Z)]. 
\end{align*}
Let 
\begin{align*}
    \hat g_h &\coloneqq \argmax_{g\in \Gcal} \tilde L_n(h,g),\quad g_h \coloneqq \EE[ Y-h(X)\mid Z=\cdot], \\ 
     \Gamma(h,g) &\coloneqq -0.5g^2(Z) +  \{Y-h(X)\}g(Z),\quad \kappa(h,g) \coloneqq \Gamma(h,g)- \Gamma(h,g_h).
\end{align*}

\paragraph{First Step.}

Our goal is to show  
\begin{align}\label{eq:first_goal}
   \forall h \in \Hcal; |\EE_n[\kappa(h,\hat g_h)]|\lesssim \frac{(C^2_{\Hcal} + C^2_{\Gcal})\ln(|\Gcal|/\delta)}{n}. 
\end{align}
We fix $h$ hereafter. 

Here, first, we have 
\begin{align*}
    \EE[ \kappa(h,\hat g_h)] = 0.5 \EE[(\hat g_h-g_h)^2(Z)].  
\end{align*}
Then, 
\begin{align*}
    \EE[\kappa(h,\hat g_h) ] & \leq \EE_n[\kappa(h,\hat g_h)] + |(\EE-\EE_n)[\kappa(h,\hat g_h) ]| \\ 
    & \leq |(\EE-\EE_n)[\kappa(h,\hat g_h) ]|.
\end{align*} 
From the first line to the second line, we use the definition of the estimator and $g_h \in \Gcal$. 

Now, we use Bernstein's inequality. With probability $1-\delta$, we have 
\begin{align*}
   \forall g \in \Gcal, \forall h \in \Hcal; (\EE-\EE_n)[\kappa(h,g) ]\leq \sqrt{ \frac{\mathrm{var}[ \kappa(h,g)]\ln( |\Gcal||\Hcal|/\delta)}{n} } + \frac{(C^2_{\Hcal} + C^2_{\Gcal})\ln(|\Gcal||\Hcal| /\delta)}{n}.
\end{align*}
In the following, we condition on this event. Then, we have  
\begin{align}\label{eq:key1}
    \EE[\kappa(h,\hat g_h) ] & \lesssim \sqrt{\frac{\mathrm{var}[\kappa(h,\hat g_h)]\ln(|\Gcal||\Hcal|/\delta)}{n}}   + \frac{(C^2_{\Hcal} + C^2_{\Gcal})\ln(|\Gcal||\Hcal|/\delta)}{n}.
\end{align}
Here, we have 
\begin{align*}
    \mathrm{var}[\kappa(h,\hat g_h) ] &= \EE[\{\Gamma (h, g_h)-\Gamma(h,\hat g_h) \}^2] \\
    &=\EE[0.25 \{\hat g_h(Z)-g_h(Z)\}^2\{\hat g_h(Z) + g_h(Z)\}^2   ] \\
    &\leq C^2_{\Gcal} \EE[\{\hat g_h(Z)-g_h(Z)\}^2 ]. 
\end{align*}
Therefore, combining the above with \pref{eq:key1}, we obtain
\begin{align*}
    \EE[\{\hat g_h(Z)-g_h(Z)\}^2 ] \lesssim \frac{(C^2_{\Hcal} + C^2_{\Gcal})\ln(|\Gcal| |\Hcal| /\delta)}{n}.
\end{align*}
Hence,
\begin{align*}
    |\EE_n[\kappa(h,\hat g_h)]| & \leq |\EE[\kappa(h,\hat g_h)]|+     |(\EE_n - \EE)[\kappa(h,\hat g_h)]|\\
    &= 0.5\EE[\{\hat g_h(Z)-g_h(Z)\}^2 ] + |(\EE_n - \EE)[\kappa(h,\hat g_h)]| \\
    &\lesssim \frac{(C^2_{\Hcal} + C^2_{\Gcal})\ln(|\Gcal||\Hcal| /\delta)}{n} +  |(\EE_n - \EE)[\kappa(h,\hat g_h)]|\\
    &\lesssim \frac{(C^2_{\Hcal} + C^2_{\Gcal})\ln(|\Gcal||\Hcal| /\delta)}{n}. \tag{Use \pref{eq:key1}}
\end{align*}

\paragraph{Second Step.}

We define 
\begin{align*}
    \Xi(h) := \Gamma(h,g_h) - \Gamma(h_0,g_{h_0}). 
\end{align*}
Note $\Gamma(h_0,g_{h_0})=0$ since $h_0$. Furthermore, 
\begin{align}\label{eq:tool2}
    \EE[\Xi(h) ] = \EE[g^2_h(Z)], \quad \EE[\Xi^2(h) ] \leq (C^2_{\Hcal} +C^2_{\Gcal})\EE[g^2_h(Z) ]. 
\end{align}

Then, 
\begin{align*}
    \EE[ \Xi(\hat h) ]\leq \EE_n[\Xi(\hat h) ] +  |(\EE-\EE_n)[\Xi(\hat h) ]| 
\end{align*}
Here, using the first conclusion \pref{eq:first_goal}, we get
\begin{align*}
    \EE_n[\Xi(\hat h) ] &= \EE_n[\Gamma(\hat h, g_{\hat h}) - \Gamma(h_0,0)] \\ 
    &\leq \EE_n[\Gamma(\hat h, \hat g_{\hat h}) - \Gamma(h_0,\hat g_{h_0})] + c\frac{(C^2_{\Hcal} + C^2_{\Gcal})\ln(|\Gcal||\Hcal| /\delta)}{n}\\
    &\leq c\frac{(C^2_{\Hcal} + C^2_{\Gcal})\ln(|\Gcal||\Hcal| /\delta)}{n}. 
\end{align*}
From the first line to the second line, we use $h_0 \in \Hcal$ and \pref{eq:first_goal}. From the second line to the third line, we use the construction of the estimator. 

Therefore, 
\begin{align*}
    \EE[ \Xi(\hat h) ]\leq c\frac{(C^2_{\Hcal} + C^2_{\Gcal})\ln(|\Gcal||\Hcal| /\delta)}{n} +  |(\EE-\EE_n)[\Xi(\hat h) ]|. 
\end{align*}
Here, we use Bernstein's inequality. With probability $1-\delta$, we have 
\begin{align*}
    \forall h \in \Hcal; |(\EE-\EE_n)[\Xi(h)] |\leq \sqrt{\frac{\mathrm{var}[\Xi(h)]\ln(|\Hcal|/\delta)}{n} } + c\frac{(C^2_{\Hcal} + C^2_{\Gcal})\ln(|\Hcal|/\delta)}{n}. 
\end{align*}
Hereafter, we condition on this event. Thus, using \eqref{eq:tool2}, we have 
\begin{align*}
    \EE[ g^2_{\hat h}(Z) ]\leq \sqrt{\frac{g^2_{\hat h}(Z)\ln(|\Hcal|/\delta)}{n} } + c\frac{(C^2_{\Hcal} + C^2_{\Gcal})\ln(|\Hcal||\Gcal|/\delta)}{n}.
\end{align*}
Therefore, by some algebra, we obtain 
\begin{align*}
    \EE[ g^2_{\hat h}(Z) ]\lesssim \frac{(C^2_{\Hcal} + C^2_{\Gcal})\ln(|\Hcal||\Gcal|/\delta)}{n}. 
\end{align*}

\section{Proof of \pref{sec:both_world}}

\subsection{Proof of \pref{thm:both}}

We use the notation in \pref{thm:main}. 
Take $\mu_n$ such that $2\Mcal(\Hcal,\Gcal)\leq \mu_n$ holds with probabiltiy $1-\delta$. We condition on this event.

The guarantee in terms of projected MSEs is straightforward as long as $h_0$ is included in the confidence ball $\Hcal_n$ with probability $1-\delta$ by following the proof in \pref{sec:discussion}. In fact, we have 
\begin{align*}
    &L_n(h_0,\hat g(h_0))- \min_{h}L_n(h,\hat g(h))=L_n(h_0,\hat g(h_0))- L_n(\hat h,\hat g(\hat h))  \\
    &\leq L_n(h_0,\hat g(h_0)) - L_n(\hat h, g(\hat h))  \\ 
    &=L_n(h_0,\hat g(h_0))-L(h_0,\hat g(h_0))) + L(h_0,\hat g(h_0))) -L(\hat h, g(\hat h))+
   L(\hat h, g(\hat h))   -L_n(\hat h, g(\hat h))  \\
   &\leq L_n(h_0,\hat g(h_0))-L(h_0,\hat g(h_0))) + L(h_0,g(h_0))) -L(\hat h, g(\hat h))+
   L(\hat h, g(\hat h))   -L_n(\hat h, g(\hat h))  \\
   &\leq 2 M(\Hcal,\Gcal)\leq \mu_n.  
\end{align*}
Hence, $h_0 \in \Hcal_n$.

Next, we prove the $L_2$ convergence guarantee. Here, for any $\hat h$ in the confidence ball $\Hcal_n$, we have 
\begin{align*}
    L_n(\hat h,\hat g(\hat h)) &\geq L_n(\hat h,\bar g_0) \tag{Construction of estimators} \\ 
    &\geq L(\hat h,\bar g_0) -  M(\Hcal,\Gcal) \\ 
    &\geq L(h_0,\bar g_0)  -  M(\Hcal,\Gcal) \tag{Saddle point property} \\ 
     &\geq L(h_0,\hat g(h_0)) -  M(\Hcal,\Gcal) \tag{Saddle point property}  \\
     &\geq L_n(h_0,\hat g(h_0)) -  2M(\Hcal,\Gcal) \\ 
    &\geq  \min_{h}L_n(h,\hat g(h)) -  2M(\Hcal,\Gcal) \\
     &\geq   L_n(\hat h,\hat g(\hat h)) -  2M(\Hcal,\Gcal)-\mu_n. 
\end{align*}
Therefore, 
\begin{align*}
    L_n(h_0,\hat g(h_0))- \min_{h}L_n(h,\hat g(h))\leq  2M(\Hcal,\Gcal) + \mu_n. 
\end{align*}
Hence, the $L_2$ rate guarantee is ensured since 
\begin{align*}
    \EE[\{\hat h(X)-h_0(X)\}^2]\leq L(\hat h,\bar g_0)-L(h_0,\bar g_0) \leq{2M(\Hcal,\Gcal) + \mu_n}. 
\end{align*}

\section{Auxiliary Lemmas}\label{sec:auxiliary}

\begin{lemma}
   $(x^{*},y^{*})$ is a saddle point of $f(x,y)$ over $(\Xcal,\Ycal)$ if and only if the strong duality holds and 
   \begin{align*}
     x^{*}  \in \argmin_{x \in \Xcal}\max_{y \in \Ycal}f(x,y), \quad 
     y^{*}  \in \argmax_{y \in \Ycal}\min_{x\in \Xcal}f(x,y). 
   \end{align*}
\end{lemma}
\begin{proof}
Suppose $(x^{*},y^{*})$ is a saddle point of $f(x,y)$ over $\Xcal,\Ycal$. Then, 
\begin{align*}
     \inf_{x\in \Xcal} \sup_{y \in \Ycal} f(x,y) \leq \sup_{y \in \Ycal} f(x^{*},y) \leq f(x^{*},y^{*})\leq \inf_{x\in \Xcal} f(x, y^{*})\leq \sup_{y \in \Ycal} \inf_{x\in \Xcal} f(x, y). 
\end{align*}
Hence, the strong duality holds. The above inequalities are actually equalities. Therefore, 
\begin{align*}
     \inf_{x\in \Xcal} \sup_{y \in \Ycal} f(x,y) = \sup_{y \in \Ycal} f(x^{*},y) = f(x^{*},y^{*}) = \inf_{x\in \Xcal} f(x, y^{*}) = \sup_{y \in \Ycal} \inf_{x\in \Xcal} f(x, y). 
\end{align*}
Hence, we have 
\begin{align*}
    x^{*} \in \argmin_{x\in \Xcal} \sup_{y \in \Ycal} f(x,y),\quad  y^{*} \in \argmax_{y \in \Ycal} \inf_{x\in \Xcal} f(x, y). 
\end{align*}

Next, suppose the strong duality holds, and 
\begin{align*}
         x^{*}  \in \argmin_{x \in \Xcal}\max_{y \in \Ycal}f(x,y), \quad 
     y^{*}  \in \argmax_{y \in \Ycal}\inf_{x\in \Xcal}f(x,y). 
\end{align*}
Then, we have 
\begin{align*}
   \max_{y\in \Ycal}\inf_{x\in \Xcal}f(x,y)= \inf_{x\in \Xcal}f(x,y^{*}) \leq f(x^{*},y^{*})\leq \sup_{y\in \Ycal} f(x^{*},y)= \min_{x\in \Xcal}\sup_{y\in \Ycal} f(x,y). 
\end{align*}
Finally, using the strong duality, the above is actually equality. Hence, 
\begin{align*}
    \inf_{x\in \Xcal}f(x,y^{*})= f(x^{*},y^{*}),\quad  \sup_{y\in \Ycal} f(x^{*},y)=f(x^{*},y^{*}). 
\end{align*}
This implies $(x^{*},y^{*})$ is a saddle point since
\begin{align*}
   \forall x\in \Xcal,\forall y\in \Ycal; f(x,y)\geq f(x^{*},y^{*})\geq f(x^{*},y). 
\end{align*}
    
\end{proof}

\begin{theorem}[Dudley integral]\label{thm:dudley}
Consider a function class $\Fcal$ containing $f:\Xcal \to \mathbb{R}$. Then, we have
\begin{align*}
    \Rcal_n(\Fcal)\leq \inf_{\epsilon\geq 0}\braces{4\epsilon + 12 \int_{\epsilon}^{\|\Fcal\|_{\infty}}\sqrt{\frac{\ln \Ncal(\tau,\Fcal,\|\cdot\|_{\infty})}{n}}\mathrm{d}\tau}.
\end{align*}
    
\end{theorem}

\end{document}